\documentclass{article}

\usepackage{microtype}
\usepackage{graphicx}
\usepackage{subcaption}

\usepackage{hyperref}

\usepackage[preprint]{icml2026}

\renewcommand{\toptitlebar}{%
  \vspace{-20pt}
  \par\noindent
  \begin{minipage}{\textwidth}
    \includegraphics[height=0.65cm]{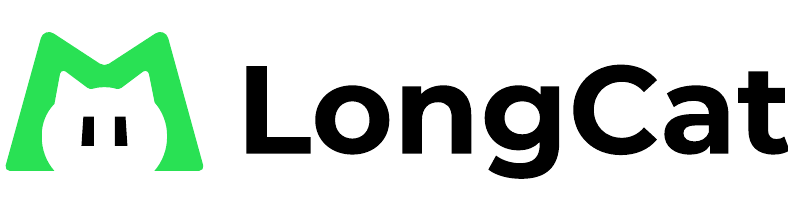}%
    \hfill
    \includegraphics[height=0.82cm]{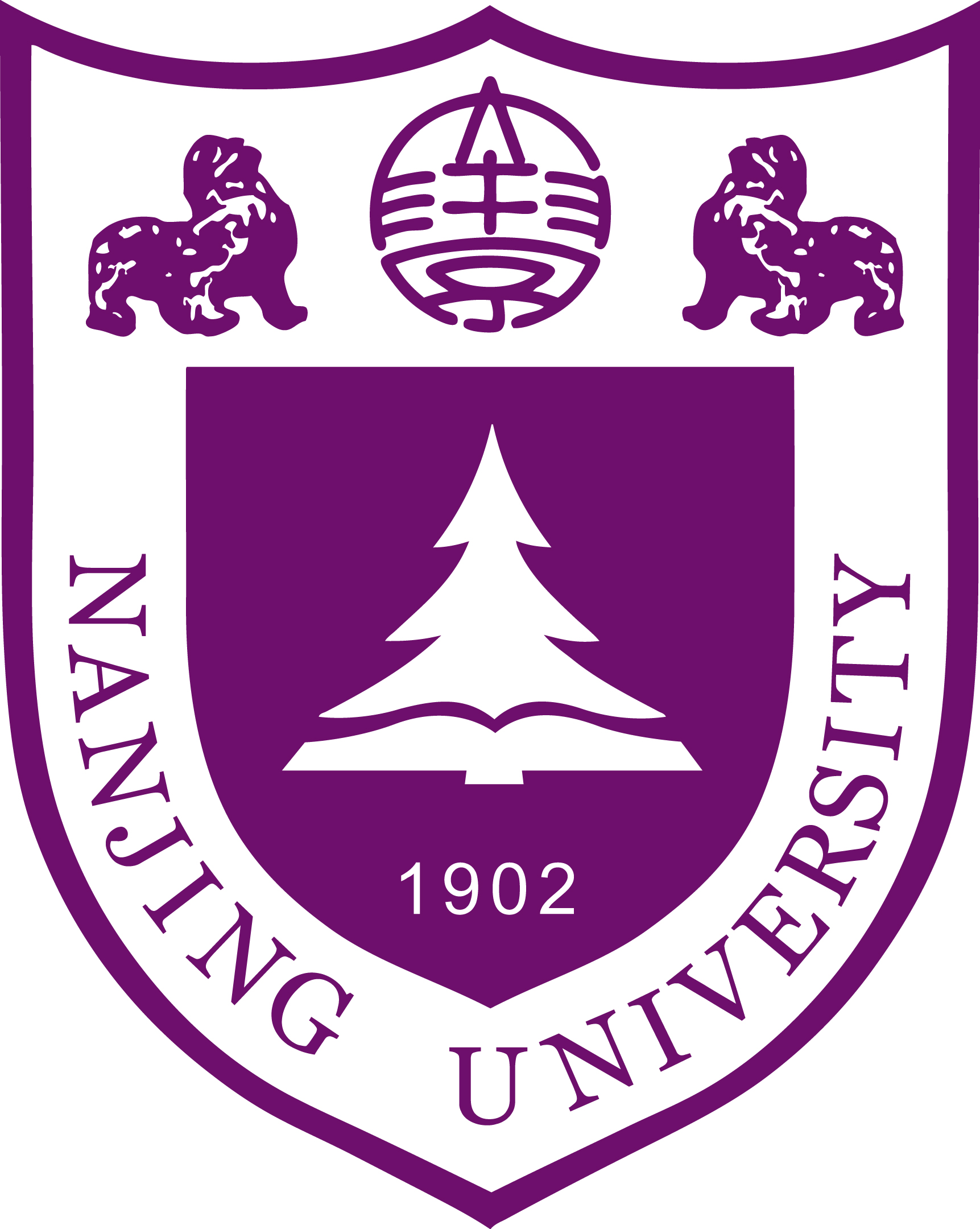}%
  \end{minipage}
  \par\vspace{5pt}
  \hrule height1pt
  \vskip .25in
}

\usepackage{caption}
\usepackage{amsmath}
\usepackage{amssymb}
\usepackage{mathtools}
\usepackage{amsthm}
\usepackage{multirow}
\usepackage{diagbox}
\usepackage{makecell}

\usepackage{times}
\usepackage{soul}
\usepackage{url}
\usepackage{booktabs}

\usepackage{times}
\usepackage{latexsym}

\usepackage[T1]{fontenc}

\usepackage{nicefrac}       
\usepackage{microtype}      
\usepackage{newunicodechar}
\usepackage{xcolor}         
\usepackage{colortbl}
\usepackage{tcolorbox}
\usepackage{enumitem}
\usepackage{fontawesome5}

\usepackage[capitalize,noabbrev]{cleveref}

\theoremstyle{plain}
\newtheorem{theorem}{Theorem}[section]

\newtheorem{lemma}[theorem]{Lemma}

\theoremstyle{definition}

\theoremstyle{remark}

\usepackage[textsize=tiny]{todonotes}

\icmltitlerunning{$V_{0.5}$: Generalist Value Model as a Prior for Sparse RL Rollouts}

\begin{document}

\icmltitle{$V_{0.5}$: Generalist Value Model as a Prior for Sparse RL Rollouts}

\begin{center}
\vspace{-10pt}

{
\textbf{Yi-Kai Zhang}\textsuperscript{1,2,3}, 
\textbf{Yueqing Sun}\textsuperscript{3},
\textbf{Hongyan Hao}\textsuperscript{3},
\textbf{Qi Gu}\textsuperscript{3,\raisebox{0.4pt}{\scriptsize\faIcon[regular]{envelope}}}, \\ 
\textbf{Xunliang Cai}\textsuperscript{3}, 
\textbf{De-Chuan Zhan}\textsuperscript{1,2},
\textbf{Han-Jia Ye}\textsuperscript{1,2,\raisebox{0.4pt}{\scriptsize\faIcon[regular]{envelope}}}
}
\vspace{5pt}

{
\textsuperscript{1}School of Artificial Intelligence, Nanjing University \\
\textsuperscript{2}National Key Laboratory for Novel Software Technology, Nanjing University \\
\textsuperscript{3}Meituan, China \\
}

\vspace{5pt}
\parbox{\linewidth}{\centering
    {Project Page:} {\small \url{https://now-join-us.github.io/V0_5}}
}
\vspace{3pt}

\end{center}

\vspace{15pt}

\begingroup
  \renewcommand{\thefootnote}{\raisebox{0.4pt}{\scriptsize\faIcon[regular]{envelope}}}
  
  \footnotetext{Correspondence to \texttt{guqi03@meituan.com} and \texttt{yehj@lamda.nju.edu.cn}.}
\endgroup

\renewcommand{\thefootnote}{\arabic{footnote}}

\begin{abstract}
In Reinforcement Learning with Verifiable Rewards (RLVR), constructing a robust advantage baseline is critical for policy gradients, effectively guiding the policy model to reinforce desired behaviors. Recent research has introduced Generalist Value Models (such as $V_0$), which achieve pre-trained value estimation by explicitly encoding model capabilities in-context, eliminating the need to synchronously update the value model alongside the policy model.
In this paper, we propose $V_{0.5}$, which adaptively fuses the baseline predicted by such value model (acting as a prior) with the empirical mean derived from sparse rollouts. This constructs a robust baseline that balances computational efficiency with extremely low variance.
Specifically, we introduce a real-time statistical testing and dynamic budget allocation. This balances the high variance caused by sparse sampling against the systematic bias (or hallucinations) inherent in the value model's prior. By constructing a hypothesis test to evaluate the prior's reliability in real-time, the system dynamically allocates additional rollout budget on demand. This mechanism minimizes the baseline estimator's Mean Squared Error (MSE), guaranteeing stable policy gradients, even under extreme sparsity with a group size of 4. Extensive evaluations across six mathematical reasoning benchmarks demonstrate that $V_{0.5}$ significantly outperforms GRPO and DAPO, achieving faster convergence and over some 10\% performance improvement.
\end{abstract}

\begin{figure}[H]
    \centering
    \includegraphics[width=\textwidth]{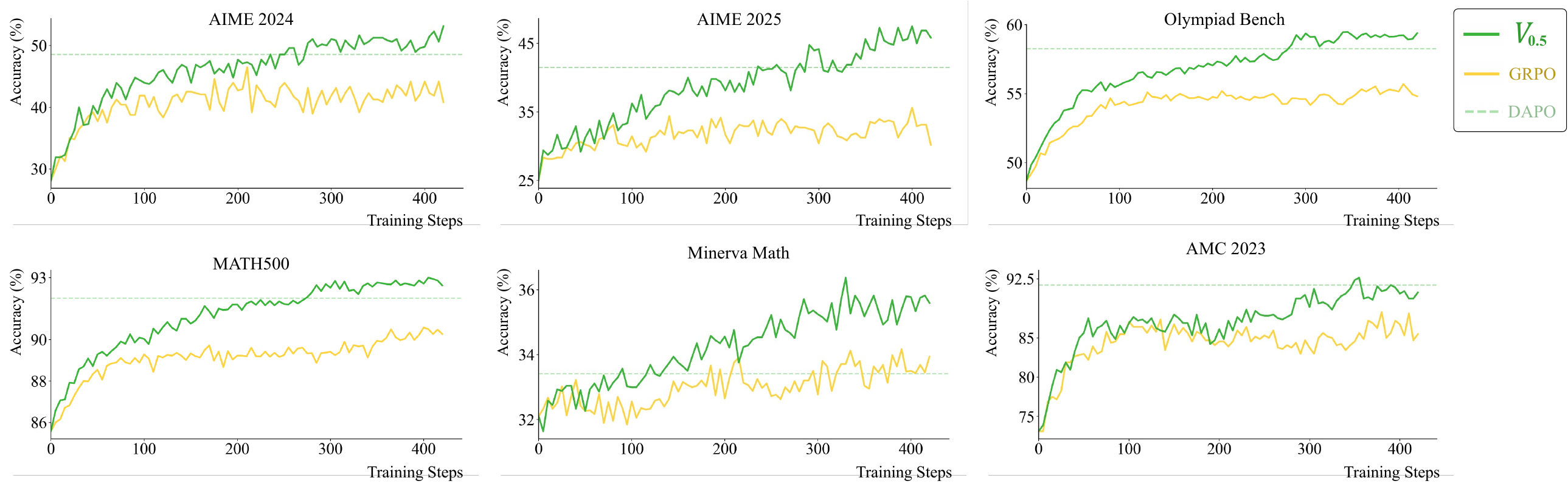}
    \caption{Performance of $\boldsymbol{V}_{\text{\textbf{0.5}}}$ across six diverse mathematical reasoning benchmarks, demonstrates superiority over GRPO~\cite{DeepSeek_Math} and DAPO~\cite{DAPO}, achieving faster convergence and some over 10\% performance improvement.}
    \label{fig:performance}
\end{figure}

\section{Introduction}

In the post-training phase of Large Language Models (LLMs), Reinforcement Learning with Verifiable Rewards (RLVR) has emerged as a standard paradigm for enhancing complex reasoning capabilities~\cite{longcat_2601,Qwen3,gemini2_5}. For policy gradient methods within this paradigm, constructing a robust advantage baseline is critical for stable training~\cite{williams1992simple,greensmith2004variance,high_dimensional_continuous_control}.
Currently, baseline estimation relies on two primary approaches: (\textbf{1}) Monte Carlo Sampling (\textit{e.g.}, GRPO~\cite{DeepSeek_R1}), which computes the empirical reward mean from online rollouts, and (\textbf{2}) Parameterized Value Models (\textit{e.g.}, PPO~\cite{ppo}), which predict the expected return via a separate model. Both face significant limitations. While empirical sampling guarantees an unbiased estimate, the high computational cost of long-horizon tasks forces the use of sparse rollouts. This sparsity exposes the empirical mean to severe statistical variance, which can compromise training stability. Conversely, while parameterized value models reduce this variance, they require expensive synchronous training and are prone to introducing systematic bias due to poor out-of-distribution (OOD) generalization.

Recent research highlights the potential of Generalist Value Models, such as $V_0$~\cite{v_0}. By explicitly encoding model capabilities in-context, $V_0$ is capable of estimating the expected performance of any model on unseen prompts without requiring parameter updates. Rather than synchronously training a coupled critic, it generates predictions via a completely separate, frozen model. Because this value estimation is generated independently of, and prior to, the actual online rollouts, it functions as a statistical prior for the advantage baseline. This prior may provide immediate guidance under sparse rollout conditions. However, like any generalist model, $V_0$ is susceptible to hallucinations or biased predictions when evaluating novel or overly complex out-of-distribution (OOD) prompts. Therefore, a critical challenge emerges: how can we safely integrate this static prior into the empirical baseline estimation of sparse online rollouts, ensuring that the policy model benefits from the prior's variance-reduction properties without being corrupted by its occasional hallucinations?

In this paper, we propose $\boldsymbol{V}_{0.5}$, an adaptive baseline estimation and budget allocation framework designed to systematically resolve this dilemma. $V_{0.5}$ operates through two tightly coupled mechanisms: (\textbf{1}) it safely fuses the value model's prior with empirical sparse rollouts to construct a more robust advantage estimation; (\textbf{2}), it adaptively scales the rollout budget based on the real-time statistical conflict between the prior predictions and the online observations. The framework is designed to fully exploit the generalist prior to suppress the high variance of sparse sampling when the predictions are reliable, while strictly bounding the impact of hallucinations by isolating the prior and dynamically allocating additional budget only when statistically necessary to correct the baseline. Specifically, $V_{0.5}$ achieves this through the following components:
\begin{itemize}[noitemsep,topsep=0pt,leftmargin=*]
\item \textbf{Empirical Shrinkage Fusion:} For a fixed number of sparse rollouts, $V_{0.5}$ constructs a baseline using a shrinkage estimator that fuses the empirical mean with the prior prediction. We demonstrate that the mean squared error (MSE) of this combined baseline orthogonally decomposes into observation variance and prior bias. To protect against severe prior bias, $V_{0.5}$ employs a positive-part truncation functionally equivalent to a statistical hypothesis test. It dynamically assesses the value model's reliability: if consistent with the rollouts, the prior is heavily leveraged to suppress variance; if a severe conflict is detected (indicating a hallucination), the system swiftly isolates the prior and reverts to the empirical mean, guaranteeing a safe, bounded error~\cite{james1961estimation}.
\item \textbf{Sequential OSLA Allocation:} Relying solely on rigid sparse rollouts can lead to false rejections of an accurate prior due to the sheer randomness of limited sampling. To address this, $V_{0.5}$ extends baseline estimation into a dynamic budget allocation problem. Grounded in One-Step-Look-Ahead (OSLA) sequential analysis, the framework quantifies baseline uncertainty in real-time to derive a continuous target boundary. This allows the system to automatically decide whether to halt sampling early to save budget, or to enforce additional rollouts to resolve conflicts and correct prior hallucinations. By balancing statistical precision with marginal costs, it achieves optimal on-demand scheduling~\cite{wald2004sequential}.
\end{itemize}
Our main contributions are summarized as follows:
\begin{itemize}[noitemsep,topsep=0pt,leftmargin=*]
\item We propose $V_{0.5}$ to safely integrate generalist value priors into sparse RL rollouts. By employing an empirical shrinkage estimator coupled with the Sequential OSLA Allocation, the framework neutralizes the high variance of limited rollouts while actively safeguarding against hallucinations of the value model.
\item We provide mathematical foundations for both mechanisms. We prove that the baseline mean squared error (MSE) orthogonally decomposes to linearly suppress the overall policy gradient variance, and we establish the asymptotic optimality of our dynamic stopping rule.
\item Extensive evaluations validate the practical impact of our approach. By enabling robust, on-demand budget scheduling, $V_{0.5}$ significantly outperforms standard GRPO, achieving over a 10\% performance improvement across six diverse mathematical reasoning benchmarks.
\end{itemize}
\section{Preliminaries}

Before detailing the $V_{0.5}$ framework, this section formalizes the policy gradient objective within Reinforcement Learning with Verifiable Rewards (RLVR). We contrast the two dominant baseline estimation paradigms, coupled value models and empirical group sampling, to highlight their respective structural limitations. Subsequently, we formalize the statistical bias-variance dilemma that requires the fusion of sparse online rollouts with a generalist prior.

\subsection{Policy Gradients and the Baseline}

During LLM post-training, a policy model $\pi_\theta$ generates a response $o$ for a prompt $x \in \mathcal{D}_{\text{prompt}}$, receiving a scalar reward $r$ from a verifier or reward function. For simplicity, we assume normalized binary rewards, $r \in \{-1, 1\}$. To reduce policy gradient variance, a baseline $\mu$ is subtracted from the reward to compute the advantage ($A = r - \mu$), ensuring the model reinforces behaviors that exceed its average capability. Theoretically, the optimal baseline is the true expected return for $x$:
\begin{equation}
    \mu_{\text{true}} = \mathbb{E}_{o \sim \pi_\theta}[r \mid x]
\end{equation}
Since computing $\mu_{\text{true}}$ over the vast generation space of LLMs is computationally infeasible, the baseline should be approximated. Two primary paradigms currently exist, but both face distinct limitations in complex reasoning tasks:
\begin{itemize}[noitemsep,topsep=0pt,leftmargin=*]
\item \textbf{Parameterized Value Models} (\textit{e.g.}, PPO~\cite{ppo}): Actor-Critic architectures maintain a parameterized value function $V_\phi$ to estimate expected returns. While effective at variance reduction, this creates a coupling dilemma: the value model must be incrementally and synchronously trained to track the non-stationary target of the evolving policy $\pi_\theta$, incurring substantial computational and memory overhead.
\item \textbf{Empirical Group Sampling} (\textit{e.g.}, GRPO~\cite{DeepSeek_Math}): Value-free methods eliminate the independent value model. Instead, for a given prompt $x$, the policy generates a group of $G$ independent candidate responses, $\{o_{1}, o_{2}, \dots, o_{G}\}$. The baseline is approximated via their empirical mean:
\begin{equation}
    \bar{v}_G = \frac{1}{G}\sum_{k=1}^G r_{k}
\end{equation}
Although $\bar{v}_G$ is an unbiased estimate of $\mu_{\text{true}}$, its observation variance is inversely proportional to the group size ($\sigma_{\text{noise}}^2 \propto 1/G$). Preventing high variance or reward collapse (uniform -1 or 1 rewards) requires extensive Monte Carlo sampling. For long-horizon tasks, relying on a large $G$ is prohibitive. Consequently, under a constrained sparse rollout regime, the amplified variance causes the baseline estimate to fluctuate, yielding gradients that can destabilize training.
\end{itemize}

\subsection{Generalist Value Models: Breaking the Coupling Dilemma}

To avoid the synchronous training overhead of traditional critics, recent research has introduced Generalist Value Models, such as $V_0$. These models reframe value estimation by shifting from implicit parameter fitting to explicit In-Context Learning (ICL).
Rather than embedding policy information into latent weights, a generalist model treats the capability of an arbitrary policy $\pi$ explicitly via a context set $\mathcal{C}_\pi = \{(x_i, r_i)\}_{i=1}^N$ of historical query-performance pairs. Given a target prompt $x$, the generalist model infers the expected return dynamically:
\begin{equation}
    V = V_0(x, \, \mathcal{C}_\pi)
\end{equation}
This allows the value model to read the current capabilities of any policy and provide a baseline estimate before any online rollouts are generated, effectively decoupling value estimation from policy evolution without requiring gradient updates. 

\subsection{The Bias-Variance Tradeoff in Sparse Rollouts}

While a generalist prior provides immediate guidance and resolves the coupling dilemma, utilizing it directly as the baseline introduces a distinct limitation. Operating as a static prior, $V$ possesses an observation variance of zero; however, constrained by the generalist model's inherent generalization limits, it is susceptible to prediction errors on out-of-distribution (OOD) prompts. This introduces a systematic prior bias, quantified as:
\begin{equation}
    \Delta^2 = (V - \mu_{\text{true}})^2
\end{equation}
If left unmitigated, these prior estimation errors systematically corrupt advantage calculation. Consequently, baseline estimation under sparse rollouts is trapped in a strict statistical tradeoff:
\begin{itemize}[noitemsep,topsep=0pt,leftmargin=*]
    \item {Empirical Mean $\bar{v}_G = \frac{1}{G}\sum_{k=1}^G r_{k}$:} Unbiased ($\text{Bias} = 0$), but subject to high variance under sparse rollouts.
    \item {Prior Prediction $V = V_0(x, \mathcal{C}_\pi)$:} Stable ($\text{Variance} = 0$), but susceptible to systematic errors.
\end{itemize}
Given that the ground truth $\mu_{\text{true}}$ is unknown, relying exclusively on either the empirical mean or the potentially biased prior is suboptimal. This establishes the fundamental necessity for the $V_{0.5}$ framework. It must intelligently fuse these two estimators, leveraging the prior to suppress rollout variance while dynamically allocating additional rollout budget when statistical conflicts indicate prior inaccuracy.

\subsection{Unified Advantage Formulations}
To clearly contrast the existing paradigms and set the stage for our proposed solution, we present the unified expressions for the advantage $A_i = A(x, o_i)$ computed for a specific rollout $o_i$\footnote{While GRPO and $V_{0.5}$ further normalize by dividing by the standard deviation of the group's rewards, we omit this for clarity.}.
\begin{itemize}[noitemsep,topsep=0pt,leftmargin=*]
\item {Coupled Value Model (PPO):} $A_i = r_i - V_\phi(x)$.
\item {Empirical Sampling (GRPO):} $A_i = r_i - \left( \frac{1}{G}\sum_{k=1}^G r_k \right)$.
\item {Our Proposed $V_{0.5}$:}
\begin{equation}
    A_i = r_i - \underbrace{\left( w \cdot \bar{v}_k + (1 - w) \cdot V_0(x, \mathcal{C}_\pi) \right)}_{\text{Fused Shrinkage Baseline } \mu^*}
\end{equation}
\end{itemize}
The $V_{0.5}$ estimates the baseline through a convex combination of the empirical mean and the generalist prior. In \autoref{sec:method}, we will rigorously detail the mathematical motivation for this specific formulation, prove how it minimizes the Mean Squared Error (MSE) of the baseline estimator, and explain how the adaptive weight $w$ is dynamically computed in real-time.
\section{Method}
\label{sec:method}

\begin{figure*}[t]
    \centering
    \includegraphics[width=0.88\textwidth]{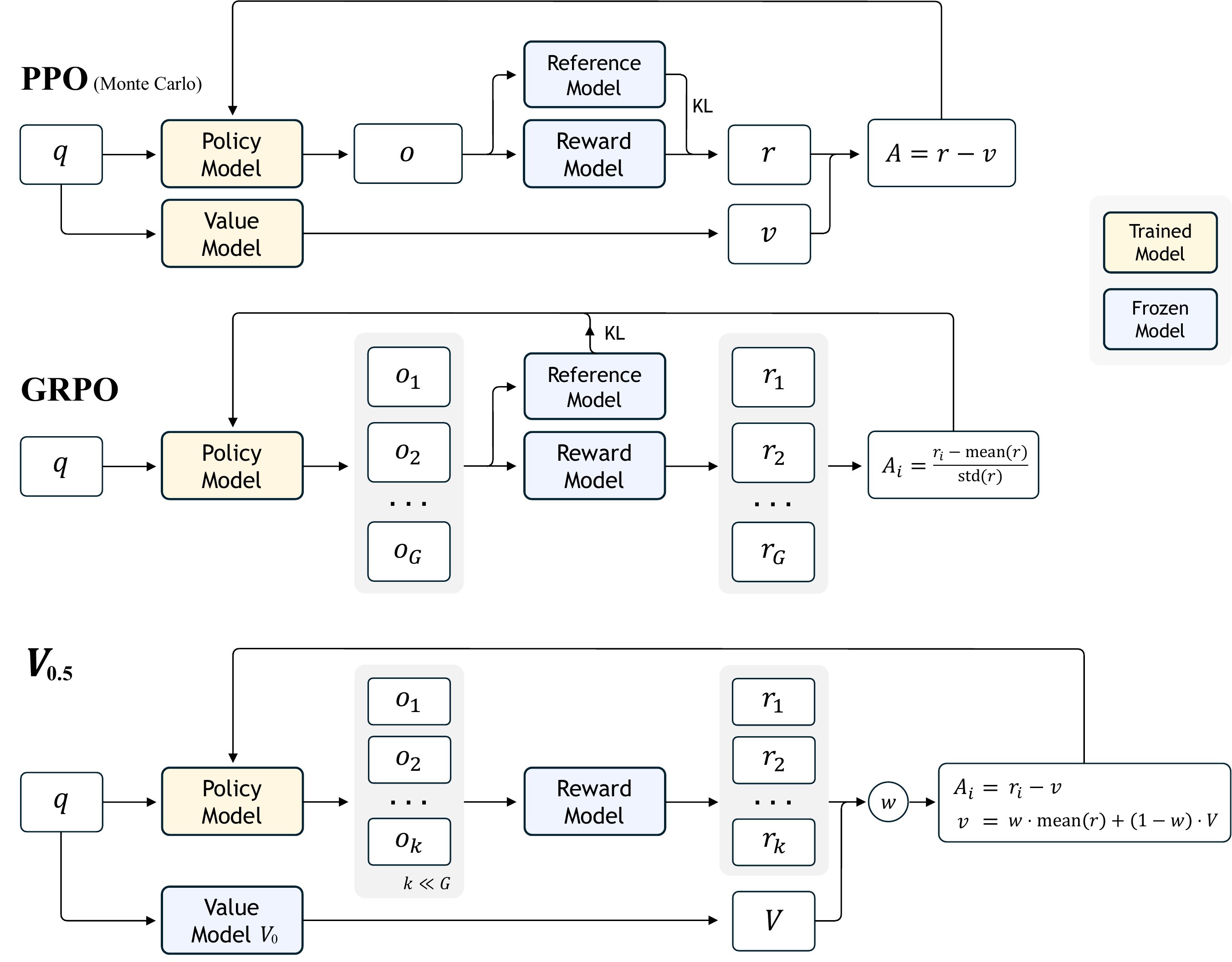}
    \caption{\textbf{Demonstration of PPO, GRPO, and our proposed $V_{0.5}$ framework.} While PPO requires a synchronously trained value model and GRPO relies on the empirical group mean, $V_{0.5}$ computes an adaptive baseline by fusing a prior from a frozen generalist value model ($V_0$) with sparse empirical rollouts via a dynamic weight $w$ (detailed in \autoref{thm:opt_weight}, \autoref{eq:sigma}, and \autoref{eq:delta}).}
    \label{fig:method}
\end{figure*}

The $V_{0.5}$ aims to resolve the limitations of empirical mean and prior prediction as baselines in sparse rollout scenarios through statistical inference and sequential decision-making mechanisms. This section first outlines the core execution logic of $V_{0.5}$, and then delves into the mathematical motivation and theoretical derivations supporting this logic.

\subsection{Core Execution Logic of $V_{0.5}$}
\label{subsec:core_logic}
The logic of $V_{0.5}$ transforms baseline estimation into a dynamic process of statistical testing and budget allocation. For a given input prompt, the system operates according to the following workflow:
\begin{enumerate}[noitemsep,topsep=0pt,leftmargin=*]
    \item Before generating rollouts, the generalist value model $V_0$ outputs a prior prediction $V$ as the expected return.
    \item The policy model generates a small initial set of $k_{\text{init}}$ rollouts, and calculates the empirical mean $\bar{v}_k$.
    \item \textbf{Deviation testing \& fusion}: The system calculates the guaranteed variance estimate $\hat{\sigma}_{\text{noise}}^2$ according to \autoref{eq:sigma} and the empirical bias $\hat{\Delta}_k^2$ according to \autoref{eq:delta}. Next, it computes the adaptive fusion weight $\hat{w}_k$ using \autoref{eq:weight}, and fuses the empirical mean with the prior prediction via \autoref{eq:fusion} to obtain the current baseline estimate.
    \item \textbf{Dynamic budget allocation}: Based on the current empirical bias $\hat{\Delta}_k^2$, the system evaluates the baseline in real-time and determines the subsequent action:
    \begin{itemize}
    \item \textbf{Stop}: If the fused baseline satisfies the criteria of the optimal stopping rule in \autoref{eq:stop}, the system halts the rollout process and outputs the advantage.
    \item \textbf{Rollout More}: If the test reveals a significant bias, the system forces the policy model to generate additional rollouts and loops back to \texttt{Step} \texttt{3} to recalculate the fused baseline. As this process iterates, the system may progressively shift its reliance toward the empirical mean.
    \end{itemize}
\end{enumerate}

Next, we detail the mathematical motivation and derivations that drive this logic.

\subsection{Motivation: Propagation Limits of Baseline MSE on Gradient Variance}
\label{subsec:motivation_baseline_mse}

Before detailing our fusion algorithm, we must clarify the theoretical optimization objective: \textbf{how does the baseline's MSE amplify policy gradient variance?} To this end, we establish the analytical relationship between the theoretical MSE of a baseline estimator $b$, denoted as $\text{MSE}(b) = \mathbb{E}[(b - \mu_{\text{true}})^2]$, and the overall variance of the policy gradient.

\begin{tcolorbox}[
colback=gray!3,
colframe=black,
boxrule=0.5pt,
arc=2pt,
left=6pt,
right=6pt,
top=4pt,
bottom=4pt
]
\begin{theorem} \label{thm:mse_bound}
For a single-step policy gradient estimator $\hat{g}(\theta)$ using baseline $b$, the trace of its covariance matrix is strictly bounded by:
$$\text{Tr}(\text{Var}(\hat{g}(\theta))) \leqslant \text{Var}_{\text{oracle}} + \Phi_{\text{score}} \cdot \text{MSE}(b) + L \cdot |\text{Bias}(b)|$$
where:
\begin{itemize}[noitemsep,topsep=0pt,leftmargin=*]
\item $\text{Var}_{\text{oracle}}$: The theoretical minimum variance achieved by a perfect baseline ($b = \mu_{\text{true}}$).
\item $\Phi_{\text{score}}$: The expected squared norm of the policy's score function ($\mathbb{E}_\pi[\|\nabla_\theta \log \pi_\theta(y|x)\|^2]$).
\item $L$: A constant scaling the cross-perturbation penalty introduced by the baseline's bias.
\end{itemize}
\end{theorem}
\end{tcolorbox}

\textbf{Analysis and Insights}: For LLMs with billions of parameters, the gradient sensitivity scale $\Phi_{\text{score}}$ is inherently massive. \autoref{thm:mse_bound} directly reveals that under sparse rollouts, any estimation error ($\text{MSE}(b)$) is severely amplified by $\Phi_{\text{score}}$, causing the gradient variance to explode. This dictates the core design philosophy of $V_{0.5}$: rather than chasing an unbiased but highly noisy empirical baseline, we must safely integrate a prior to intentionally \textbf{tolerate a minor, mathematically bounded bias $|\text{Bias}(b)|$ in exchange for a massive reduction in $\text{MSE}(b)$}.

\subsection{Empirical Shrinkage Fusion}

To minimize the theoretical $\text{MSE}$ under a fixed number of $k$ sparse samples, we construct a Shrinkage Estimator that fuses the empirical mean $\bar{v}_k$ and the prior prediction $V$ via a convex combination: $\mu^* = w \bar{v}_k + (1 - w) V$. By substituting this into the MSE, we establish the following theorem.
\begin{tcolorbox}[
colback=gray!3,
colframe=black,
boxrule=0.5pt,
arc=2pt,
left=6pt,
right=6pt,
top=4pt,
bottom=4pt
]
\begin{theorem} \label{thm:mse_decomp}
Since the prior $V$ is deterministic before rollout and $\bar{v}_k$ is an unbiased estimator, the MSE of this shrinkage estimator can be decomposed into a weighted sum of empirical variance and prior bias:
$$\text{MSE}(w) = \mathbb{E}[(\mu^* - \mu_{\text{true}})^2] = w^2 \sigma_{\text{noise}}^2 + (1 - w)^2 \Delta^2$$
\end{theorem}
\end{tcolorbox}
With the error decomposed, we can pinpoint the exact weight that minimizes the overall MSE, as shown in \autoref{thm:opt_weight}.
\begin{tcolorbox}[
colback=gray!3,
colframe=black,
boxrule=0.5pt,
arc=2pt,
left=6pt,
right=6pt,
top=4pt,
bottom=4pt
]
\begin{theorem} \label{thm:opt_weight}
Setting the derivative of $\text{MSE}(w)$ with respect to $w$ to zero yields the unique global optimal weight that minimizes the estimation error:
$$w^* = \frac{\Delta^2}{\Delta^2 + \sigma_{\text{noise}}^2}$$
\end{theorem}
\end{tcolorbox}

\subsubsection{Empirical Weight Estimation}
In practice, the theoretical prior bias $\Delta^2$, the true variance $\sigma_{\text{noise}}^2$, and consequently the optimal weight $w^*$ from \autoref{thm:opt_weight} are inaccessible. We must approximate them using real-time observations to construct an empirical estimator. We denote these with a hat (\textit{e.g.}, $\hat{\sigma}_{\text{noise}}^2$, $\hat{\Delta}_k^2$, $\hat{w}_k$).

\begin{enumerate}[noitemsep,topsep=0pt,leftmargin=*]
\item Since rewards are bounded in $\left\{ -1, 1 \right\}$, the variance of a single rollout is at most $1$. Thus, the variance of the empirical mean $\bar{v}_k$ over $k$ independent rollouts is bounded by $1/k$, which we use as our guaranteed estimate:
\begin{equation} \label{eq:sigma}
    \hat{\sigma}_{\text{noise}}^2 = \frac{1}{k}
\end{equation}
\item The empirical estimate of the prior bias, $\hat{\Delta}_k^2$, is derived from the observed squared distance $(\bar{v}_k - V)^2$:
\begin{equation} \label{eq:delta}
\hat{\Delta}_k^2 = \max\left(0, (\bar{v}_k - V)^2 - \frac{1}{k}\right)
\end{equation}
\textbf{Statistically, this $\max$ operator functions as a simplified hypothesis test against random noise}: it strictly attributes the discrepancy to random sampling variations unless the observation distance significantly breaches the maximum noise boundary $1/k$. For the detailed statistical derivation establishing this equivalence, please refer to \autoref{app:truncation_2_hypothesis_testing}.
\end{enumerate}
Finally, by directly substituting these empirical estimates into the optimal formulation from \autoref{thm:opt_weight}, we derive the practical adaptive weight $\hat{w}_k$:
\begin{equation} \label{eq:weight}
\hat{w}_k = \frac{\hat{\Delta}_k^2}{\hat{\Delta}_k^2 + \hat{\sigma}_{\text{noise}}^2}
\end{equation}
This naturally yields our final fused empirical baseline estimator:
\begin{equation} \label{eq:fusion}
\hat{\mu}^* = \hat{w}_k \bar{v}_k + (1 - \hat{w}_k)V
\end{equation}
Let $\rho_i(\theta) = \frac{\pi_\theta(o_i \mid x)}{\pi_{\theta_{\text{old}}}(o_i \mid x)}$ denote the importance sampling ratio. Utilizing the advantage $A_i = \frac{r_i - \hat{\mu}^*}{\hat{\sigma}^*}$ with $\hat{\sigma}^* = \sqrt{1 - \hat{\mu}^{*2}}$, our surrogate optimization objective for a group size $k$ is formulated as:
\begin{equation} \label{eq:objective}
    \mathcal{J}_{V_{0.5}}(\theta) = \mathbb{E}_{x \sim \mathcal{D}_{\text{prompt}}, \{o_i\}_{i=1}^{k} \sim \pi_{\theta_{\text{old}}}} \left[ \frac{1}{k} \sum_{i=1}^{k} \min \left( \rho_i(\theta) A_i,\, \operatorname{clip}\left(\rho_i(\theta), 1-\epsilon, 1+\epsilon\right) A_i \right) \right]
\end{equation}

\subsubsection{Bias Bounds of the Estimator}

Unlike $\mu^*$ which uses the optimal weight $w^*$, our empirical weight $\hat{w}_k$ is dependent on the random variable $\bar{v}_k$. This statistical correlation renders the fused estimator $\hat{\mu}^*$ biased. However, \autoref{thm:bias_bound} guarantees that this induced bias is strictly confined within safe analytical limits.
\begin{tcolorbox}[
    colback=gray!3,
    colframe=black,
    boxrule=0.5pt,
    arc=2pt,
    left=6pt,
    right=6pt,
    top=4pt,
    bottom=4pt
]
\begin{theorem} \label{thm:bias_bound}
The empirical baseline estimator $\hat{\mu}^*$ possesses the following safety properties:
\begin{enumerate}
    \item $|\text{Bias}(\hat{\mu}^*)| \leqslant \frac{1}{\sqrt{k}}$.
    \item If $\Delta \neq 0$, the induced bias decays at $\mathcal{O}(\frac{1}{k})$ as the $k$ increases.
\end{enumerate}
\end{theorem}
\end{tcolorbox}
Recalling \autoref{thm:mse_bound}, we observe that by allowing a bounded bias of at most $\mathcal{O}(1/\sqrt{k})$, the $V_{0.5}$ estimator effectively offsets the catastrophic $\mathcal{O}(1/k)$ amplification of the MSE, ensuring stable training gradients even under extreme sparsity.

\subsection{Sequential OSLA Allocation and Optimal Stopping}
While static fusion optimally balances bias and variance under fixed compute, relying purely on a rigid sample size under extreme sparsity can result in high observational noise, potentially leading to false rejections of an accurate prior. To resolve this, $V_{0.5}$ integrates a dynamic sequential One-Step-Look-Ahead (OSLA) allocation mechanism.
We define the total risk after $k$ sampling steps as a combination of the empirical estimation error and the compute cost:
\begin{equation}
R(k) = \widehat{\text{MSE}}(k) + c \cdot k \quad \text{where} \quad \widehat{\text{MSE}}(k) = \frac{\hat{\Delta}_k^2}{k \hat{\Delta}_k^2 + 1}
\end{equation}
Here, $c$ denotes the compute cost factor per rollout. Crucially, $\widehat{\text{MSE}}(k)$ represents the \textit{empirical} MSE, calculated by substituting the unknown $\Delta^2$ with our real-time empirical estimate $\hat{\Delta}_k^2$. The decision to allocate additional budget hinges on whether the expected reduction in this empirical error exceeds the marginal cost.

\begin{tcolorbox}[
    colback=gray!3,
    colframe=black,
    boxrule=0.5pt,
    arc=2pt,
    left=6pt,
    right=6pt,
    top=4pt,
    bottom=4pt
]
\begin{lemma} \label{lem:marginal_return}
To derive a safe, continuous analytical boundary for deployment, we scale the discrete marginal return to establish its lower bound envelope:
\begin{equation*}
\widehat{\text{MSE}}(k) - \widehat{\text{MSE}}(k+1) > \frac{\hat{\Delta}_k^4}{((k+1) \hat{\Delta}_k^2 + 1)^2}
\end{equation*}
\end{lemma}
\end{tcolorbox}

By equating this lower bound of the expected statistical return from \autoref{lem:marginal_return} to the marginal cost, we guarantee that executing an additional rollout yields a positive benefit, which derives the threshold for our dynamic budget allocation.
\begin{tcolorbox}[
    colback=gray!3,
    colframe=black,
    boxrule=0.5pt,
    arc=2pt,
    left=6pt,
    right=6pt,
    top=4pt,
    bottom=4pt
]
\begin{theorem} \label{thm:optimal_stop}
Given the compute cost $c$ and the updated empirical bias $\hat{\Delta}_k^2$, and assuming a minimum initial sample size $k_{\text{min}}$, the optimal stopping step $K^*$ is defined as follows:
\begin{equation} \label{eq:stop}
K^* = \inf \left\{ k \geqslant k_{\text{min}} : k \geqslant \frac{1}{\sqrt{c}} - \frac{1}{\hat{\Delta}_k^2} \right\}
\end{equation}
\end{theorem}
\end{tcolorbox}
In practice, the rollout generation starts with a small initial size $k_{\text{init}}$ and evaluates the optimal stopping condition in real-time, yielding a dynamic rollout budget $K^*_i$ for each prompt. By substituting the group size $k$ in \autoref{eq:objective} with $K^*_i$, the policy update is driven entirely by robust, sparse rollouts that are adaptively tailored to the empirical uncertainty of each prompt.
\section{Experiments}
\label{sec:experiments}

\subsection{Experimental Setup and Implementation Details}

\subsubsection{Introduction to the Generalist Value Model ($V_0$)}
The Generalist Value Model ($V_0$) avoids the synchronous gradient updates characteristic of traditional Actor-Critic architectures. Instead of implicitly fitting the policy's evolving capabilities, $V_0$ explicitly represents these capabilities using a context of historical query-performance pairs, denoted as $\mathcal{C}_\pi = \{(x_i, r_i)\}$. Consequently, the traditional value estimation $V^\pi(x)$ is reformulated as $V(\mathcal{C}_\pi, x)$. This paradigm allows the value model to be fully pre-trained offline, providing zero-gradient advantage baselines during the RL phase. The network architecture of $V_0$ consists of:

\begin{figure*}[t]
    \centering
    \includegraphics[width=0.8\textwidth]{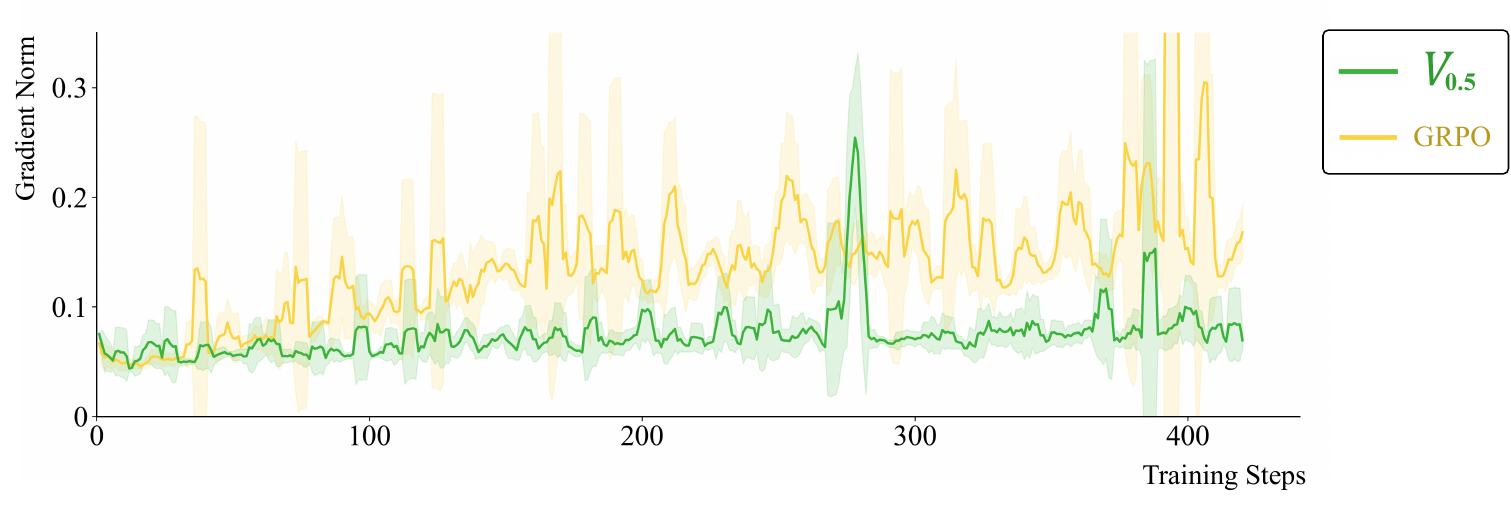}
    \caption{\textbf{Evolution of policy gradient norm.} $V_{0.5}$ maintains a lower and more stable gradient norm than GRPO. By trading a strictly bounded bias for a reduced baseline MSE, it effectively neutralizes the variance amplification inherent in sparse rollouts.}
    \label{fig:grad_norm}
\end{figure*}

\begin{figure*}[t]
    \centering
    \includegraphics[width=0.8\textwidth]{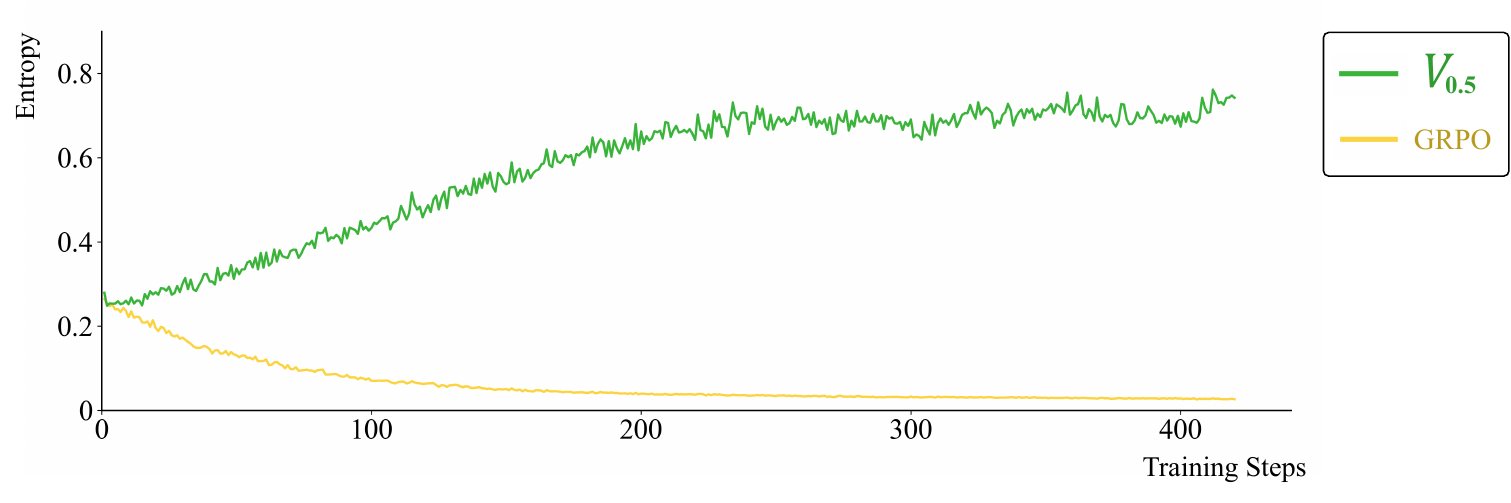}
    \caption{\textbf{Evolution of policy entropy.} While GRPO's high-variance gradients cause rapid entropy decay, $V_{0.5}$ leverages low-noise baseline estimation to sustain higher entropy, ensuring robust exploration in reasoning tasks.}
    \label{fig:entropy}
\end{figure*}

\begin{itemize}[noitemsep,topsep=0pt,leftmargin=*]
    \item \textbf{Semantic-Perception Backbone}: Uses an LLM embedding model to map instructions into semantic vectors.
    \item \textbf{Residual Query Adapter}: Employs learnable queries and a residual mechanism to project the entangled semantic features into a structured, compressed latent space.
    \item \textbf{Probabilistic In-Context Head}: Uses TabPFN to perform single-pass Bayesian inference on the historical pairs $\mathcal{C}_\pi$, yielding the predicted success probability of the target query.
\end{itemize}

\subsubsection{Enhanced Training of $V_0$}
To support complex mathematical reasoning tasks, we trained an enhanced $V_0$ model.

\paragraph{Data Construction.} The value model requires diverse performance data to capture a wide range of capabilities. We sampled GRPO training trajectories from LLMs across various architectural scales. Each trajectory contained over 200 checkpoints, with about 20$k$ rollouts per checkpoint. Building upon the models used for the original $V_0$ ({Qwen3-4B-Instruct-2507}, {Qwen2.5-7B-Instruct}, and {DeepSeek-R1-Distill-Qwen-1.5B}), we expanded the pool to include the full Qwen3 series (0.6B to 30B parameters), with Base, Instruct, and Thinking variants. This data synthesis produced approximately 424$k$ high-quality training pairs reflecting rich evolutionary histories and performance disparities.

\paragraph{Model Architecture and Hyperparameters.} We retained the base $V_0$ architecture. The backbone utilizes the frozen {Qwen3-Embedding-0.6B} ($d_{\text{embed}} = 1024$). The adapter is configured with 168 static queries, a projection dimension of 6, and 3 Multi-Head Attention (MHA) layers. The inference head employs {TabPFN-v2.5}. During training, we randomly sample 256 query-performance pairs to form the capability context $\mathcal{C}_\pi$. The enhanced model was pre-trained on 128 GPUs for approximately 40 hours. All other hyperparameters (\textit{e.g.}, learning rate, loss ratios) remain identical to the original $V_0$.

\subsubsection{Algorithmic Workflow and Implementation of $V_{0.5}$}
The $V_{0.5}$ framework achieves low variance and adaptive compute scheduling in sparse rollout scenarios by tightly integrating \textbf{Empirical Shrinkage Fusion} with \textbf{Sequential OSLA Allocation}.

\begin{figure*}[t!]
    \centering
    \includegraphics[width=\textwidth]{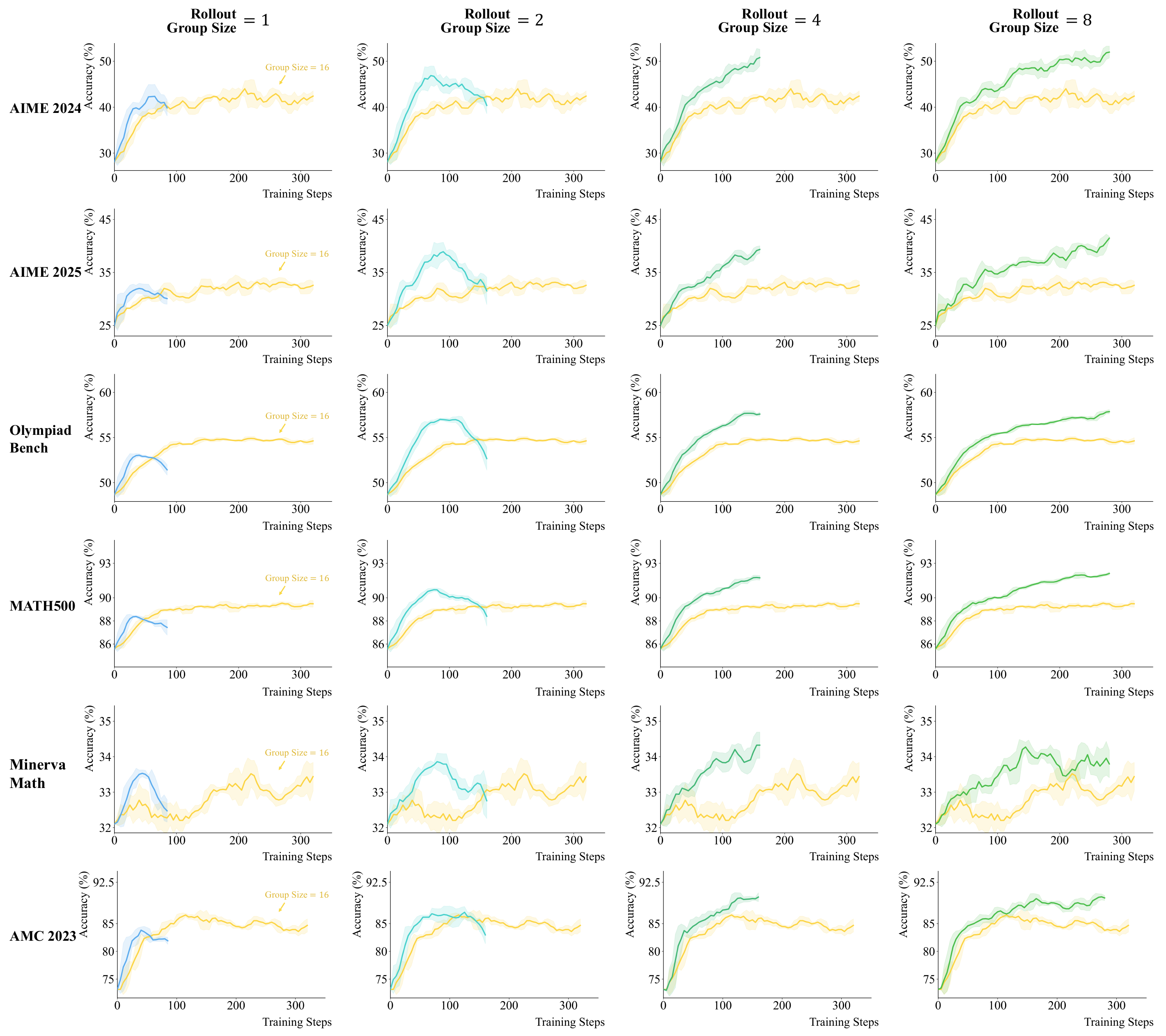}
    \caption{\textbf{Performance of $V_{0.5}$ under extreme sparsity} ($1, 2, 4,$ and $8$ rollouts) \textit{vs.} standard GRPO ($16$ rollouts). To ensure a fair comparison, prompt batch sizes are adjusted to maintain constant per-step computational overhead.}
    \label{fig:main_results}
\end{figure*}

\begin{itemize}[noitemsep,topsep=0pt,leftmargin=*]
    \item \textbf{Step 1: Construct Context and Acquire Prior $V$}. The system maintains a global Support Buffer (capacity 512) storing recent sample performances ($\mathcal{C}_\pi$). It randomly samples 256 pairs (Support Batch Size) and queries the value model API in batches (Query Batch Size) to retrieve the target query's predicted success probability $p \in [0, 1]$.
    \item \textbf{Step 2: Cold Start Allocation}. An initial baseline rollout is executed for all batch prompts. To ensure robust hypothesis testing in the discrete binary reward space $\{-1, 1\}$, we set the initial group size to $k_{\text{init}} = 4$. As detailed in \autoref{app:base_group_size}, this size guarantees the test's tolerance radius fully absorbs the observation gap caused by discrete sampling.
    \item \textbf{Step 3: Initial Bias Evaluation and Hypothesis Testing}. The system uses the empirical mean $\bar{v}_4$ to calculate the squared observation bias $(\bar{v}_4 - V)^2$ against the theoretical noise upper bound $1/4$. If $(\bar{v}_4 - V)^2 \leqslant 1/4$, the prior $V$ is deemed highly reliable ($\hat{\Delta}_k^2 = 0$), the deviation is attributed to observation noise, and no extra compute is required. Conversely, if $\hat{\Delta}_k^2 > 0$, it indicates potential value model hallucination, triggering additional budget allocation.
    \item \textbf{Step 4: Sequential OSLA Allocation and Alignment}. For samples exhibiting bias, the system calculates a compute boundary. Given a marginal step cost $c = 0.0039$, the maximum budget is $1/\sqrt{c} \approx 16$. The dynamic target compute is defined as $k_{\text{target}} = 1/\sqrt{c} - 1/\hat{\Delta}_k^2$. If the current observation count $k < k_{\text{target}}$, $2$ additional rollouts are allocated.
    \begin{itemize}
        \item To maximize distributed hardware utilization, dynamic generation halts globally if fewer than $25\%$ of the batch samples require additional compute.
        \item To prevent resource fragmentation during Tensor Parallelism, the allocated budget for a single dispatch is automatically padded to a multiple of $32$.
    \end{itemize}
    \item \textbf{Step 5: Final Empirical Shrinkage Fusion and Advantage Calculation}. Once dynamic allocation concludes, the system collects the final $k$ observations ($4 \le k \le 16$) per prompt. It updates the empirical mean $\bar{v}_k$, the guaranteed variance estimate $\hat{\sigma}_{\text{noise}}^2$, and the empirical bias $\hat{\Delta}_k^2$. The system then calculates the adaptive shrinkage weight $\hat{w}_k$ to derive the low-variance fused baseline $\hat{\mu}^*$ and deduces the intrinsic standard deviation $\hat{\sigma}^* = \sqrt{1 - (\hat{\mu}^*)^2}$. The final standardized advantage, $A_i = (r_i - \hat{\mu}^*) / \hat{\sigma}^*$, is passed to the surrogate objective function to update the policy.
\end{itemize}

\subsubsection{Training Hyperparameters, Baseline Configurations, and Evaluation Metrics}

\paragraph{Training Hyperparameters.} We conducted all RL training on 4 nodes (32 GPUs total) using the \texttt{sglang} engine. The base policy is {Qwen3-4B-Instruct-2507}, fine-tuned on the DAPO-Math-17k dataset~\cite{DAPO}. The Actor uses the AdamW optimizer with a learning rate of $1 \times 10^{-6}$, $10$ warmup steps, no KL divergence penalty, no learning rate decay, and a global gradient clipping of $1.0$. For both training and evaluation, the maximum response length is $4096$.
To ensure fair comparisons, we maintained a constant total computational overhead per step across all settings. Specifically, the product of the global prompt batch size and the group size per prompt is held constant.

\paragraph{Baseline Configurations.} All baselines operate under identical hardware and software frameworks:
\begin{itemize}[noitemsep,topsep=0pt,leftmargin=*]
    \item {GRPO}: Uses a global prompt batch size of $512$, a fixed group size $G=16$, and a KL penalty coefficient of $0.001$.
    \item {DAPO}: Uses the same batch size of $512$, group size $G=16$, and introduces DAPO-specific advantage filtering and asymmetric clipping. The rollout generation batch size is $512 \times 3$.
\end{itemize}

\paragraph{Evaluation Metrics} We evaluated performance across six mathematical reasoning benchmarks: AIME 2024~\cite{aime2024}, AIME 2025~\cite{aime2025}, Olympiad Bench~\cite{OlympiadBench}, MATH500~\cite{math500}, Minerva Math~\cite{minerva_math}, and AMC 2023~\cite{amc23}. We report the average dataset accuracy using a fixed sampling size of 16 (\texttt{mean@16}). Both training and evaluation rely strictly on rule-based rewards for correctness verification, omitting formatting rewards. Evaluation inference uses a temperature of $1.0$ and a Top-p of $0.7$.
\begin{itemize}[noitemsep,topsep=0pt,leftmargin=*]
    \item \autoref{fig:performance}: Shows the performance of $V_{0.5}$ with the complete OSLA dynamic budget allocation enabled ($k_{\text{init}} = 4$).
    \item \autoref{fig:main_results}: Compares $V_{0.5}$ without OSLA (using fixed group sizes of $1$, $2$, $4$, and $8$) against standard GRPO ($G=16$).
\end{itemize}

\subsection{Main Results and Theoretical Validation}
\label{subsec:main_results}

This section analyzes the comprehensive performance of the $V_{0.5}$ framework under the OSLA dynamic budget allocation setting ($k_{\text{init}}=4$). Furthermore, we dissect its core advantages in gradient stability and policy exploration through the lens of our theoretical derivations.

\paragraph{Overall Performance and Convergence Speed.}
As illustrated in \autoref{fig:performance}, we evaluated $V_{0.5}$ with full OSLA dynamic budget allocation across six diverse mathematical reasoning benchmarks. The results demonstrate the significant superiority of $V_{0.5}$: compared to GRPO and DAPO, $V_{0.5}$ not only achieves faster convergence but also attains over a 10\% improvement in final accuracy. By introducing a generalist value model as a prior, the policy model receives high-quality advantage estimation guidance early in the sparse rollouts, substantially enhancing sample efficiency and final performance.

\paragraph{Mathematical Mechanism of Gradient Norm and Variance Reduction.}
\autoref{fig:grad_norm} illustrates the evolution of the policy gradient norm during training. It is clearly observable that, compared to GRPO, the $V_{0.5}$ framework maintains a lower and more stable gradient norm.
This phenomenon is explained by the \autoref{thm:mse_bound} (proof in \autoref{app:mse_bound}).
\begin{itemize}[noitemsep,topsep=0pt,leftmargin=*]
    \item In LLM training, the expected squared norm of the score function, $\Phi_{\text{score}}$, is inherently massive. GRPO relies solely on extremely sparse empirical sampling, leading to an exceptionally high $\text{MSE}(b)$. This high error is severely amplified by the massive $\Phi_{\text{score}}$, inevitably causing gradient variance explosion and norm oscillation.
    \item $V_{0.5}$ employs Empirical Shrinkage Fusion to integrate the prior $V$, strategically trading a strictly bounded minor bias for a reduction in variance. It intentionally tolerates a mathematically constrained bias ($|\text{Bias}(\hat{\mu}^*)| \le \frac{1}{\sqrt{k}}$) to exchange for a drastic reduction in $\text{MSE}(b)$. This design neutralizes the catastrophic amplification effect caused by $\Phi_{\text{score}}$, guaranteeing the numerical stability of gradients under extreme sparsity and high noise.
\end{itemize}

\paragraph{Policy Entropy Maintenance and Exploration Capability}
The effective reduction of gradient variance not only stabilizes the training process but also directly enhances the policy's exploration behavior. \autoref{fig:entropy} shows the evolution of policy entropy over training steps.
Under sparse rollout scenarios, GRPO's policy entropy decays rapidly. This occurs because high-variance, erroneous gradient signals may force the model into local optima. Conversely, benefiting from its low-variance, low-noise gradient estimation, the $V_{0.5}$ framework sustains a higher entropy level throughout the training cycle. This ensures that the model can robustly conduct exploration within the complex mathematical reasoning space.

\paragraph{Performance under Extreme Sparsity and Limitations of Group Size.} In \autoref{fig:main_results}, we tested sparse group sizes ($k \in \{1, 2, 4, 8\}$) without OSLA dynamic allocation. Notably, simply fusing the prior $V$ with sparse rollouts ($k \in \{4, 8\}$) is sufficient to outperform standard GRPO ($G=16$), strongly validating the core variance reduction of our shrinkage estimator. However, training fails to converge at extreme sparsity ($k \in \{1, 2\}$). This aligns with our base group size derivations (\autoref{app:base_group_size}). In the binary reward space $\{-1, 1\}$, our hypothesis test uses a theoretical maximum noise bound ($1/\sqrt{k}$) to absorb normal observation fluctuations. For $k \in \{1, 2\}$, the large discrete quantization gaps between possible empirical means exceed this tolerance radius. Consequently, normal sampling variance frequently triggers false rejections of the prior $V$. Discarding this stabilizing prior in favor of a highly noisy empirical mean induces severe gradient variance.

\section{Related Work}

In RL for LLMs, the estimation quality of the advantage baseline directly determines the variance and stability of policy gradient updates. To eliminate the substantial computational overhead introduced by auxiliary value models in traditional PPO~\cite{ppo,VAPO}, GRPO~\cite{DeepSeek_Math} proposes directly utilizing the average reward of intra-group samples as an empirical baseline. Concurrently, ReMax~\cite{remax} replaces the value network with rewards obtained from the model's own greedy decoding, while OPO~\cite{opo} derives an optimal baseline approximation using response-length-weighted average rewards, grounded in the assumption of gradient orthogonality.

However, standard empirical means can be sensitive to outlying rewards, often failing in sparse-rollout scenarios. To address this, MC-GRPO~\cite{mc_grpo} introduces the median baseline and the Median Absolute Deviation (MAD) to resist outliers that trigger advantage sign flipping. QAE~\cite{qae} designs a $K$-quantile dual-state gated baseline to filter noise and prevent entropy explosion, whereas BNPO~\cite{bnpo} models expected rewards as a Beta distribution, dynamically computing optimal normalization parameters via moment estimation. Further research reveals that relying solely on intra-group sampling introduces statistical biases. For instance, HA-DW~\cite{ha_dw} demonstrates that group means underestimate the advantage of difficult prompts, thus introducing Kalman filter-based history-aware anchors.

Moreover, targeting ambiguous credit assignment in long-horizon or structured tasks, Turn-PPO~\cite{turn_ppo} advocates aligning actions and states to physical turns to provide turn-level baselines. GiGPO~\cite{gigpo} designs a bi-level relative advantage baseline combining global normalization with micro-anchor grouping, and Tree-OPO~\cite{tree_opo} formulates advantage computation as a constrained quadratic programming problem that respects the topological logic of Monte Carlo trees.
Despite significant advancements in lightweight design, robustness, and bias control, these methods may remain constrained by high variance and estimation bias in extreme sparse-sampling environments when relying exclusively on empirical data. Our $V_{0.5}$ elegantly resolves this dilemma by introducing the Generalist Value Model ($V_0$) as an explicit prior. It innovatively incorporates real-time statistical testing and dynamic One-Step-Look-Ahead (OSLA) sequential budget allocation to adaptively fuse sparse empirical means with prior predictions.

\section{Conclusion}

In this paper, we introduced the $V_{0.5}$ framework, representing the next-generation evolution of generalist value models. By seamlessly uniting adaptive baseline estimation with dynamic budget allocation, this framework robustly optimizes advantage estimation with the generalist value model as a statistical prior. We demonstrate that the proposed Empirical Shrinkage Fusion effectively minimizes the baseline's Mean Squared Error (MSE). Safeguarded by real-time hypothesis testing, $V_{0.5}$ bounds the negative effects of prior hallucinations, ensuring stable policy gradient convergence even under extreme sparsity with a group size of merely 4. Furthermore, our Sequential OSLA Allocation mechanism reframes baseline estimation as a continuous dynamic scheduling problem, enabling the automatic, on-demand adjustment of rollout budgets. Extensive evaluations across six diverse mathematical reasoning benchmarks establish that $V_{0.5}$ outperforms both GRPO and DAPO, achieving faster convergence and some over 10\% performance improvement.
Looking ahead, we aim to construct and pre-train a Process-level Generalist Value Model. By providing finer-grained guidance signals for trajectories, we expect substantial breakthroughs in enhancing the exploration efficiency of increasingly complex, long-horizon tasks.

\section*{Acknowledgments}
The authors would like to thank Ying Zeng from Peking University for the valuable discussions on the theoretical analysis.

\section*{Impact Statement}
This paper presents work whose goal is to advance the field of Machine Learning. There are many potential societal consequences of our work, none which we feel must be specifically highlighted here.

\bibliography{example_paper}
\bibliographystyle{icml2026}

\newpage
\appendix
\onecolumn

\section*{Appendix}

\section{Theoretical Analysis and Proofs}

\subsection{Extended Statistical Motivation (\autoref{subsec:core_logic})}

\textbf{Objective:} In \autoref{subsec:core_logic}, we construct a pipeline consisting of prior prediction, sparse rollouts, deviation testing and fusion, and dynamic budget allocation. The underlying logic of this design addresses the statistical tradeoff inherent in the \textit{exploration-exploitation dilemma} within the continuous decision spaces.

Under extreme sparse rollouts scenarios, although the empirical mean $\bar{v}_k$ is unbiased, its signal-to-noise ratio is exceedingly low. If the system relies exclusively on $\bar{v}_k$ to determine whether to allocate additional rollouts, it is highly susceptible to being misled by the variance of a single extreme rollout, potentially leading to unbounded computational consumption. The primary objective of introducing the prior $V$ in $V_{0.5}$ is to establish a zero-variance statistical anchor without incurring additional rollout budgets. Subsequent mechanisms, including shrinkage fusion and hypothesis testing, essentially quantify the statistical significance of the current empirical distribution's deviation from this anchor. When the deviation significantly exceeds the theoretical noise boundary, the system identifies the anchor as invalid, activating the protocol for additional budget allocation.

\subsection{Bounding Policy Gradient Variance with Baseline MSE (Proof of \autoref{thm:mse_bound})}
\label{app:mse_bound}

\textbf{Objective:} In \autoref{subsec:motivation_baseline_mse}, we posit that minimizing the MSE of the baseline estimator $\mu$ is mathematically equivalent to suppressing the unbounded gradient variance during LLM training. We prove that the trace of the policy gradient covariance matrix is strictly bounded by the baseline MSE.

\begin{proof}[Proof of \autoref{thm:mse_bound}]
Let the single-step policy gradient estimator be defined as $\hat{g}(\theta) = \nabla_\theta \log \pi_\theta(o|x) (r - \mu)$. Because introducing any baseline $\mu$ independent of the current action $o$ preserves the unbiasedness of the policy gradient (\textit{i.e.}, $\mathbb{E}[\hat{g}(\theta)] = g_{\text{true}}$), the trace of its covariance matrix is formulated as:
\begin{equation}
\text{Tr}(\text{Var}(\hat{g}(\theta))) = \mathbb{E}_{o,\mu}[\|\hat{g}(\theta)\|^2] - \|g_{\text{true}}\|^2
\end{equation}
Minimizing the variance is analytically equivalent to minimizing the second moment of the gradient. We introduce the true expected return $\mu_{\text{true}}$ into the squared advantage term $(r - \mu)^2$ and apply algebraic expansion:
\begin{equation}
(r - \mu)^2 = (r - \mu_{\text{true}})^2 + (\mu_{\text{true}} - \mu)^2 + 2(r - \mu_{\text{true}})(\mu_{\text{true}} - \mu)
\end{equation}
Given that the score function $\nabla_\theta \log \pi_\theta(o|x)$ depends exclusively on the policy's intrinsic sampling distribution and remains independent of the externally estimated baseline $\mu$, we substitute the expanded equation back into the second moment and decompose it via linearity into three distinct components:

\textbf{1. Oracle Variance Term} (Theoretical Lower Bound):
\begin{equation}
\mathbb{E}_o \left[ \|\nabla_\theta \log \pi_\theta(o|x)\|^2 (r - \mu_{\text{true}})^2 \right]
\end{equation}
This term represents the irreducible variance when the baseline is perfectly predicted ($\mu = \mu_{\text{true}}$), defined as $\text{Var}_{\text{oracle}}$.

\textbf{2. MSE Propagation Term} (Variance Amplifier):
\begin{equation}
\mathbb{E}_{o,\mu} \left[ \|\nabla_\theta \log \pi_\theta(o|x)\|^2 (\mu_{\text{true}} - \mu)^2 \right]
\end{equation}
This component strictly factorizes into the expected squared norm of the score function $\Phi_{\text{score}} = \mathbb{E}_o[\|\nabla_\theta \log \pi_\theta(o|x)\|^2]$ multiplied by the baseline $\text{MSE}(\mu)$. For models comprising billions of parameters, $\Phi_{\text{score}}$ is exceptionally large; consequently, any marginal baseline error undergoes severe proportional amplification.

\textbf{3. Bias Cross-Perturbation Term} (Controlled Penalty):
\begin{equation}
2 \mathbb{E}_{o,\mu} \left[ \|\nabla_\theta \log \pi_\theta(o|x)\|^2 (r - \mu_{\text{true}})(\mu_{\text{true}} - \mu) \right] = 2 \mathbb{E}_o[\dots] \cdot (-\text{Bias}(\mu))
\end{equation}
Defining the absolute value of the trace of the policy's intrinsic constant matrix as $L/2$, the magnitude of this cross term is rigorously bounded by $L \cdot |\text{Bias}(\mu)|$.

Synthesizing the three terms yields the comprehensive upper bound for the policy gradient variance:
\begin{equation}
\text{Tr}(\text{Var}(\hat{g}(\theta))) \leqslant \text{Var}_{\text{oracle}} + \Phi_{\text{score}} \cdot \text{MSE}(\mu) + L \cdot |\text{Bias}(\mu)|
\end{equation}
This establishes the theoretical validity of the framework, proving that trading a marginally controlled bias for a substantial reduction in MSE is mathematically optimal.
\end{proof}

\subsection{Orthogonal Error Decomposition and Optimal Shrinkage Weight (Proof of \autoref{thm:mse_decomp} and \autoref{thm:opt_weight})}

\textbf{Objective:} To prove that the MSE of the shrinkage estimator $\mu^* = w\bar{v}_k + (1 - w)V$ admits an orthogonal decomposition, and to derive the closed-form optimal weight $w^*$.

\begin{proof}[Proof of \autoref{thm:mse_decomp}]
By definition, the mean squared error is $\text{MSE}(w) = \mathbb{E}[(\mu^* - \mu_{\text{true}})^2]$. Substituting the estimator formulation and partitioning the constant $\mu_{\text{true}}$ via $1 = w + (1 - w)$ yields:
\begin{equation}
\text{MSE}(w) = \mathbb{E}[ (w(\bar{v}_k - \mu_{\text{true}}) + (1 - w)(V - \mu_{\text{true}}))^2 ]
\end{equation}
Expanding the quadratic term:
\begin{equation}
\text{MSE}(w) = w^2 \mathbb{E}[(\bar{v}_k - \mu_{\text{true}})^2] + (1 - w)^2(V - \mu_{\text{true}})^2 + 2w(1 - w)(V - \mu_{\text{true}})\mathbb{E}[\bar{v}_k - \mu_{\text{true}}]
\end{equation}
Because $\bar{v}_k$ operates as an unbiased estimator (\textit{i.e.}, $\mathbb{E}[\bar{v}_k - \mu_{\text{true}}] = 0$), the cross term evaluates strictly to 0. Substituting the observational variance $\sigma_{\text{noise}}^2$ and the prior bias $\Delta^2$ results in the orthogonal decomposition:
\begin{equation}
\text{MSE}(w) = w^2 \sigma_{\text{noise}}^2 + (1 - w)^2 \Delta^2
\end{equation}
\end{proof}

\begin{proof}[Proof of \autoref{thm:opt_weight}]
Differentiating the decomposed MSE objective with respect to $w$ and setting the derivative to 0:
\begin{equation}
\frac{d \text{MSE}(w)}{d w} = 2w \sigma_{\text{noise}}^2 - 2(1 - w)\Delta^2 = 0
\end{equation}
Solving this linear equation yields the unique stationary point $w^* = \frac{\Delta^2}{\Delta^2 + \sigma_{\text{noise}}^2}$. The second derivative is $2\sigma_{\text{noise}}^2 + 2\Delta^2 > 0$, guaranteeing this point is the global minimum.
\end{proof}

\subsection{Equivalence of Truncation to Hypothesis Testing}
\label{app:truncation_2_hypothesis_testing}
\textbf{Objective:} To formalize the statistical properties of the empirical bias estimator $\hat{\Delta}_k^2 = \max(0, (\bar{v}_k - V)^2 - 1/k)$. We demonstrate that this positive-part truncation functionally maps to a statistical hypothesis test with a defined rejection region.

During the fusion of prior knowledge and empirical observations, the primary uncertainty lies in whether the value model's prior prediction $V$ contains systematic hallucinations (\textit{i.e.}, true $\Delta^2 > 0$). We define the null hypothesis as the absence of systematic error in the prior:
\begin{itemize}[noitemsep,topsep=0pt,leftmargin=180pt]
\item $H_0: \Delta^2 = 0$ (\textit{i.e.}, $V = \mu_{\text{true}}$)
\item $H_1: \Delta^2 > 0$
\end{itemize}
Under a maximum entropy assumption within a binary reward space, the theoretical variance upper bound for a single rollout is exactly 1. Consequently, the variance of the empirical mean $\bar{v}_k$ is tightly bounded by $\sigma_{\text{noise}}^2 \leqslant 1/k$. Assuming the null hypothesis $H_0$ holds, the expected squared distance between the empirical observation and the prior, $(\bar{v}_k - V)^2$, must logically equal the pure observation noise:
\begin{equation}
\mathbb{E}[(\bar{v}_k - V)^2 \mid H_0] = \mathbb{E}[(\bar{v}_k - \mu_{\text{true}})^2] = \sigma_{\text{noise}}^2 \leqslant 1/k
\end{equation}
The truncation operation $\max(0, \cdot)$ executing in real-time precisely mirrors hypothesis testing logic:

\begin{enumerate}[noitemsep,topsep=0pt,leftmargin=*]
\item \textbf{Acceptance Region:} If the divergence satisfies $(\bar{v}_k - V)^2 \leqslant 1/k$, the system attributes this fluctuation to random noise. Lacking evidence to reject $H_0$, truncation activates ($\hat{\Delta}_k^2 = 0$). The system discards surface-level error, relying on the prior to minimize variance.
\item \textbf{Rejection Region:} If the divergence violates the boundary $(\bar{v}_k - V)^2 > 1/k$, the realized error exceeds any theoretically plausible statistical noise, enforcing a strict rejection of $H_0$. Truncation is bypassed, and the system linearly subtracts the expected noise baseline $(\bar{v}_k - V)^2 - 1/k$ to isolate the true magnitude of the underlying prior bias.
\end{enumerate}

\subsection{Finite-Sample and Asymptotic Bias Bounds of the Shrinkage Estimator (Proof of \autoref{thm:bias_bound})}

\textbf{Objective:} In practical deployment, the adaptive weight $\hat{w}_k = \frac{\hat{\Delta}_k^2}{\hat{\Delta}_k^2 + 1/k}$ acts as a random variable correlated with the empirical observation $\bar{v}_k$. This correlation violates the unbiasedness assumption of static shrinkage. This section quantifies the induced bias $\text{Bias}(\hat{\mu}^*) = \mathbb{E}[\hat{\mu}^*] - \mu_{\text{true}}$ and proves the two fundamental safety bounds defined in \autoref{thm:bias_bound}.

\begin{proof}[Proof of \autoref{thm:bias_bound}]
We reconstruct the empirical estimator $\hat{\mu}^*$ to isolate its unbiased component from the bias-inducing correction term:
\begin{equation}
\hat{\mu}^* = \hat{w}_k \bar{v}_k + (1 - \hat{w}_k)V = \bar{v}_k - (1 - \hat{w}_k)(\bar{v}_k - V)
\end{equation}
Taking the mathematical expectation of both sides and subtracting the objective expectation $\mu_{\text{true}}$ yields an expected error of 0 for the leading term, as $\bar{v}_k$ is strictly unbiased. Thus, all systemic bias is isolated within the correction term:
\begin{equation}
\text{Bias}(\hat{\mu}^*) = -\mathbb{E}\left[ (1 - \hat{w}_k)(\bar{v}_k - V) \right]
\end{equation}
Substituting the analytical formula for the adaptive weight, we define the inner random entity as variable $Z$:
\begin{equation}
Z = \frac{1/k}{\hat{\Delta}_k^2 + 1/k} (\bar{v}_k - V)
\end{equation}
The proof objective fundamentally reduces to bounding the absolute value of $Z$.

\textbf{Property 1 in \autoref{thm:bias_bound}:}
Next, we leverage the piecewise nature of the positive-part truncation in $\hat{\Delta}_k^2$ to exhaustively analyze arbitrary rollout across the entire sample space:

\begin{enumerate}[noitemsep,topsep=0pt,leftmargin=*]
\item \textbf{Case 1} (No truncation penalty): When $(\bar{v}_k - V)^2 \leqslant 1/k$, truncation enforces $\hat{\Delta}_k^2 = 0$. Substituting this into $Z$ simplifies the expression to $Z = \bar{v}_k - V$. Applying the square root to the piecewise condition directly yields:
\begin{equation}
|Z| \leqslant \frac{1}{\sqrt{k}}
\end{equation}

\item \textbf{Case 2} (Truncation penalty triggered): When $(\bar{v}_k - V)^2 > 1/k$, truncation is inactive, maintaining $\hat{\Delta}_k^2 = (\bar{v}_k - V)^2 - 1/k$. The denominator of $Z$ algebraically reduces to exactly $(\bar{v}_k - V)^2$. Thus:
\begin{equation}
Z = \frac{1/k}{(\bar{v}_k - V)^2} (\bar{v}_k - V) = \frac{1/k}{\bar{v}_k - V}
\end{equation}
Given the active condition $|\bar{v}_k - V| > 1/\sqrt{k}$, its reciprocal rigorously satisfies $\frac{1}{|\bar{v}_k - V|} < \sqrt{k}$. Multiplying by the fixed numerator $1/k$ produces:
\begin{equation}
|Z| < \frac{1/k}{1/\sqrt{k}} = \frac{1}{\sqrt{k}}
\end{equation}
\end{enumerate}

Synthesizing both cases, the truncation architecture mathematically guarantees that $|Z| \leqslant 1/\sqrt{k}$ almost surely across the probability space. Applying the triangle inequality for expectations establishes the absolute bound:
\begin{equation}
|\text{Bias}(\hat{\mu}^*)| \leqslant \mathbb{E}[|Z|] \leqslant \frac{1}{\sqrt{k}}
\end{equation}
\textbf{Theoretical Guarantee:} Regardless of severe observation noise encountered during extreme sparse rollouts, the systemic bias introduced by dynamic fusion is constrained within a safe upper envelope, precluding numerical gradient explosion.

\textbf{Property 2 in \autoref{thm:bias_bound}:}
Assuming an objective prior bias exists ($\Delta = |V - \mu_{\text{true}}| \neq 0$), we introduce a minimal error tolerance $\epsilon = \Delta/2 > 0$. Leveraging Hoeffding's inequality, we partition the probability space into two mutually exclusive events:
\begin{enumerate}[noitemsep,topsep=0pt,leftmargin=*]
\item \textbf{Case 1} (typical event, $|\bar{v}_k - \mu_{\text{true}}| \leqslant \epsilon$):
The observation converges closely to the objective truth. By the reverse triangle inequality, the surface difference $|\bar{v}_k - V|$ is bounded below:
\begin{equation}
|\bar{v}_k - V| \geqslant |V - \mu_{\text{true}}| - |\bar{v}_k - \mu_{\text{true}}| \geqslant \Delta - \frac{\Delta}{2} = \frac{\Delta}{2}
\end{equation}
Given sufficient compute $k$ (\textit{i.e.}, $1/\sqrt{k} < \Delta/2$), the system deterministically stabilizes in Case 2. Scaling $|Z|$ correspondingly yields:
\begin{equation}
|Z| = \frac{1/k}{|\bar{v}_k - V|} \leqslant \frac{1/k}{\Delta/2} = \frac{2}{k\Delta}
\end{equation}
Thus, the expected bias contribution from typical events is strictly upper-bounded by $\mathcal{O}(1/k)$.

\item \textbf{Case 2} (Rare Event $A^c$, $|\bar{v}_k - \mu_{\text{true}}| > \epsilon$):
According to Hoeffding's inequality, the probability of such severe deviations decays exponentially with $k$ rollouts, expressed as $\mathbb{P}(A^c) = \mathcal{O}(e^{-ck})$. Because $|Z|$ remains globally bounded by $1/\sqrt{k}$, the expected contribution from rare events is:
\begin{equation}
\mathbb{E}[|Z| \mid A^c] \cdot \mathbb{P}(A^c) \leqslant \frac{1}{\sqrt{k}} \cdot \mathcal{O}(e^{-ck})
\end{equation}
\end{enumerate}
Because the exponential decay $\mathcal{O}(e^{-ck})$ approaches zero substantially faster than the polynomial decay $1/k$, the total asymptotic limit is entirely dictated by typical events, proving:
\begin{equation}
|\mathbb{E}[\hat{\mu}^*] - \mu_{\text{true}}| = \mathcal{O}\left(\frac{1}{k}\right)
\end{equation}
\textbf{Theoretical Guarantee:} While standard empirical estimation errors decay at a rate of $\mathcal{O}(1/\sqrt{k})$, the targeted systemic bias induced by the framework dissipates at a highly aggressive, super-linear rate of $\mathcal{O}(1/k)$, establishing a mathematically superior variance reduction tradeoff.
\end{proof}

\subsection{Marginal Return Envelope and Optimal Stopping Rule (Proof of \autoref{lem:marginal_return} and \autoref{thm:optimal_stop})}

\textbf{Objective:} To formally derive the lower-bound envelope of marginal statistical returns for sequential rollouts (\autoref{lem:marginal_return}) and to establish the continuous optimal stopping boundary for adaptive budget allocation (\autoref{thm:optimal_stop}).

\begin{proof}[Proof of \autoref{lem:marginal_return}]
The real-time empirical MSE after $k$ steps is expressed as:
\begin{equation}
\widehat{\text{MSE}}(k) = \frac{\hat{\Delta}_k^2}{k \hat{\Delta}_k^2 + 1}
\end{equation}
Defining the single-step marginal return $g(k)$ as the expected reduction in empirical MSE from one additional rollout, and assuming local smoothness under the One-Step-Look-Ahead (OSLA) framework (\textit{i.e.}, $\hat{\Delta}_{k+1}^2 \approx \hat{\Delta}_k^2$), we have:
\begin{equation}
g(k) = \widehat{\text{MSE}}(k) - \widehat{\text{MSE}}(k+1) = \frac{\hat{\Delta}_k^2}{k \hat{\Delta}_k^2 + 1} - \frac{\hat{\Delta}_k^2}{(k+1) \hat{\Delta}_k^2 + 1} = \frac{\hat{\Delta}_k^4}{(k \hat{\Delta}_k^2 + 1)((k+1) \hat{\Delta}_k^2 + 1)}
\end{equation}
Since $k \geqslant 0$ and $\hat{\Delta}_k^2 \geqslant 0$, the denominator factors inherently satisfy $(k \hat{\Delta}_k^2 + 1) < ((k+1) \hat{\Delta}_k^2 + 1)$. Replacing the smaller polynomial with the larger one, generating a lower-bound envelope:
\begin{equation}
g(k) > \frac{\hat{\Delta}_k^4}{((k+1) \hat{\Delta}_k^2 + 1)^2}
\end{equation}
This ensures that if this minimal envelope exceeds the cost factor, continuing to rollout remains definitively profitable.
\end{proof}

\begin{proof}[Proof of \autoref{thm:optimal_stop}]
To locate the optimal stopping threshold, we balance the marginal return bound with the normalized marginal compute cost $c$, and taking the square root yields:
\begin{equation}
\frac{\hat{\Delta}_k^4}{((k+1) \hat{\Delta}_k^2 + 1)^2} = c, \,\,\, \frac{\hat{\Delta}_k^2}{(k+1) \hat{\Delta}_k^2 + 1} = \sqrt{c}
\end{equation}
And then, we have:
\begin{equation}
(k+1) \hat{\Delta}_k^2 + 1 = \frac{\hat{\Delta}_k^2}{\sqrt{c}}
\end{equation}
Isolating the continuous boundary variable $(k+1)$:
\begin{equation}
(k+1) = \frac{1}{\sqrt{c}} - \frac{1}{\hat{\Delta}_k^2}
\end{equation}
This closed-form formulation perfectly defines the optimal stopping threshold $K^*$ introduced in \autoref{eq:stop}.

\textbf{Application:} This formulation elegantly transitions reinforcement learning compute scheduling into a real-time feedback control mechanism. The leading term $1/\sqrt{c}$ operates as the theoretical maximum compute budget constrained by fixed costs. The subtraction term $1/\hat{\Delta}_k^2$ functions as an adaptive decay penalty driven by real-time observed bias. Severe prior hallucinations diminish the decay term, expanding the boundary toward the maximum budget; conversely, an accurate prior rapidly inflates the decay term, dynamically collapsing the target compute boundary and executing an early halt.
\end{proof}

\subsection{Finite-Cost Regret Bound of Sequential Scheduling}

\textbf{Objective:} In the dynamic budget allocation process, the system relies on the empirical estimate $\hat{\Delta}_k^2$, which contains statistical noise, to make real-time decisions on when to halt rollout. Compared to an ideal oracle system that knows the true bias $\Delta^2$ in advance and directly outputs the perfect stopping step $k_{\text{oracle}}^*$, our adaptive scheduling inevitably incurs an excess trial cost. We define this excess cost as the finite-cost regret:
\begin{equation}
\text{Regret}(c) = \mathbb{E}[R(K^*)] - R(k_{\text{oracle}}^*)
\end{equation}
where $R(k) = \widehat{\text{MSE}}(k) + c \cdot k$ is the total risk function of the system. This part proves that the upper bound of this regret is enveloped by the marginal cost itself, specifically $\mathcal{O}(c)$.

\begin{proof}
\textbf{Step 1:} Apply a Taylor expansion to the risk function to quantify the penalty curvature at the optimal stopping point.

We extend the discrete system risk function to a continuous one:
\begin{equation}
R(x) = \frac{\Delta^2}{x \Delta^2 + 1} + c x
\end{equation}
Based on the derivation in \autoref{thm:optimal_stop}, the global continuous minimum under the oracle perspective is $x^* = \frac{1}{\sqrt{c}} - \frac{1}{\Delta^2}$.
We perform a second-order Taylor expansion of the risk function at $x^*$ for the actual stopping time $K^*$. Since the first derivative at the extremum is $R'(x^*) = 0$, the regret is entirely dominated by the second-order term:
\begin{equation}
\text{Regret}(c) \approx \frac{1}{2} R''(x^*) \mathbb{E}[(K^* - x^*)^2]
\end{equation}

\textbf{Analysis:} This Taylor expansion strategically decomposes the excess cost into two components:
\begin{enumerate}[noitemsep,topsep=0pt,leftmargin=*]
\item \textbf{Curvature at the minimum $R''(x^*)$:} This represents the steepness of the risk function near the optimal point. A steeper curvature implies a more severe penalty for deviating from the optimal stopping step.
\item \textbf{Variance of the stopping step $\mathbb{E}[(K^* - x^*)^2]$:} This quantifies how far the actual stopping point deviates from the ideal point due to observation noise.
\end{enumerate}

To evaluate the curvature, we take the second derivative of the risk function:
\begin{equation}
R''(x) = \frac{2\Delta^6}{(x \Delta^2 + 1)^3}
\end{equation}
Substituting the minimum point $x^*$, the denominator simplifies to $(\Delta^2/\sqrt{c})^3$. Consequently, the curvature at this point is heavily dependent on the compute cost $c$:
\begin{equation}
R''(x^*) = 2 c^{3/2}
\end{equation}

\textbf{Step 2:} Apply the Delta method to translate empirical estimation noise into the MSE of the stopping time.

The discrepancy between the actual stopping time $K^*$ and the theoretical time $x^*$ originates from replacing the unknown true bias $\Delta^2$ with the noisy empirical bias $\hat{\Delta}_{K^*}^2$.
To translate this parameter estimation error into a decision step error, we apply the Delta method (first-order Taylor expansion) to analyze the fluctuation of the stopping step:
\begin{equation}
K^* - x^* \approx \left( \frac{1}{\sqrt{c}} - \frac{1}{\hat{\Delta}_{K^*}^2} \right) - \left( \frac{1}{\sqrt{c}} - \frac{1}{\Delta^2} \right) = \frac{1}{\Delta^2} - \frac{1}{\hat{\Delta}_{K^*}^2} \approx \frac{\hat{\Delta}_{K^*}^2 - \Delta^2}{\Delta^4}
\end{equation}
According to statistical estimation theory, near the stopping boundary (where the expected step count is $x^* \approx 1/\sqrt{c}$), the variance of the empirical variance estimator decays inversely with the rollout size:
\begin{equation}
\text{Var}(\hat{\Delta}_{K^*}^2) = \mathcal{O}\left(\frac{1}{x^*}\right) = \mathcal{O}(\sqrt{c})
\end{equation}
Therefore, the mean squared error of the stopping step induced purely by parameter estimation noise is:
\begin{equation}
\mathbb{E}[(K^* - x^*)^2] = \mathcal{O}(\sqrt{c})
\end{equation}

\textbf{Step 3:} Incorporate discrete and nonlinear corrections to establish the final upper bound for the total expected regret.

Multiplying the mean squared error $\mathcal{O}(\sqrt{c})$ obtained in Step 2 back into the Taylor expansion from Step 1 yields the magnitude of the regret caused by pure parameter estimation error:
\begin{equation}
\text{Regret}_{\text{estimation}} \approx c^{3/2} \cdot \mathcal{O}(\sqrt{c}) = \mathcal{O}(c^2)
\end{equation}
However, the continuous Taylor expansion above conceals two discrete and nonlinear penalties present in the actual system. The true total regret must superimpose the following correction terms:
\begin{enumerate}[noitemsep,topsep=0pt,leftmargin=*]
\item \textbf{Discretization Overshoot:} The actual rollout step size must be an integer. When the system crosses the continuous theoretical stopping boundary, a constant integer overshoot error with a mean squared magnitude of $\mathcal{O}(1)$ is inevitable. Multiplying this by the curvature $c^{3/2}$, the regret contributed by this term is $\mathcal{O}(c^{3/2})$.
\item \textbf{Renewal Correlation:} According to nonlinear renewal theory, the decision boundary triggering the stop action is not statistically independent of the current empirical estimator, because the current data dictates the stopping action. This covariance introduces a first-order shift. After rigorous mapping, the penalty introduced to the total risk function is strictly bounded at the $\mathcal{O}(c)$ scale.
\end{enumerate}

Synthesizing these three orthogonal terms, the overall expected regret of the system is:
\begin{equation}
\text{Regret}(c) = \mathcal{O}(c^2) + \mathcal{O}(c^{3/2}) + \mathcal{O}(c)
\end{equation}

\textbf{Analysis:} In the context of RL rollouts, the marginal compute cost $c$ of a single step is inherently a minute constant (\textit{e.g.}, $c = 0.0039$). For values strictly less than 1, we have $c \gg c^{3/2} \gg c^2$. Consequently, the lowest-order term $\mathcal{O}(c)$ dictates the entire polynomial.
Thus, the total system regret is safely enveloped by:
\begin{equation}
\text{Regret}(c) \leqslant \mathcal{O}(c)
\end{equation}
\end{proof}

\textbf{Theoretical and Engineering Significance:}
This proof establishes a robust theoretical guarantee for the deployment of the $V_{0.5}$ framework.
It mathematically demonstrates that abandoning a rigid fixed batch size and allowing the system to perform adaptive inference under extreme sparsity and high noise inherently introduces inference errors. However, the expected excess cost incurred to locate the optimal stopping point is merely equivalent to a constant number of compute rollouts (strictly proportional to $c$ itself, potentially equating to just one or two rollouts).
This powerful bound ensures that the system can aggressively pursue high-precision baseline estimation while remaining completely immune to unbounded compute cost explosions caused by dynamic decision errors.

\subsection{Analysis of the Base Group Size}
\label{app:base_group_size}
\textbf{Objective:} To analyze the constraint on the base rollout group size $k_{\text{min}}$ introduced in \autoref{thm:optimal_stop}. We derive the lower bound for the minimum initial rollout size to ensure the statistical robustness of the hypothesis testing mechanism within a discrete binary reward space.

\begin{proof}
In the preceding analysis, the truncation mechanism of the empirical bias estimator $\hat{\Delta}_k^2$ operates equivalently to a hypothesis test in a continuous space. Considering the reward space defined in this paper is a discrete binary space $\{-1, 1\}$, if the initial rollout size $k$ is overly sparse, the discrete jumps in observations will violate the smooth boundary assumptions of the hypothesis test.
Under the binary reward $r \in \{-1, 1\}$, if there are $x$ successful responses within $k$ rollouts, the empirical mean is $\bar{v}_k = \frac{2x}{k} - 1$. The discrete quantization gap between adjacent possible observation values is:
\begin{equation}
\operatorname{Gap}(k) = \frac{2}{k}
\end{equation}
Simultaneously, the testing threshold, representing the tolerance radius for accepting the prior $V$, is defined by the maximum entropy standard deviation:
\begin{equation}
\operatorname{Threshold}(k) = \frac{1}{\sqrt{k}}
\end{equation}

To guarantee the statistical robustness of the hypothesis test, the testing tolerance radius must be greater than or equal to the discrete jump gap induced by a single observation change. Failing this, the random fluctuation of a single rollout will compel the mean to bypass the entire confidence interval, inciting frequent misjudgments and system instability during high-frequency oscillations.

Assuming the value model provides a high-confidence prior $V = 0.8$. For small values of $k$, the system is compromised because the discrete gap strictly exceeds the tolerance radius:
\begin{itemize}[noitemsep,topsep=0pt,leftmargin=*]
\item {When $k = 1$:} $\text{Gap}(1) = 2$, and the testing radius is $1$. If a single rollout yields $-1$, the empirical error is $|-1 - 0.8| = 1.8 > 1$, resulting in an immediate rejection. A single rollout's noise strictly dominates the test outcome.
\item {When $k = 2$:} $\text{Gap}(2) = 1$, and the testing radius is approximately $0.707$. If a rollout yields $1$ and $-1$ outcome, the error is $|0 - 0.8| = 0.8 > 0.707$, triggering a rejection. Standard variance systematically invalidates the prior.
\item {When $k = 3$:} $\text{Gap}(3) \approx 0.667$, and the testing radius is approximately $0.577$. The discrete gap remains larger than the tolerance radius, perpetuating an unstable testing state.
\item {When} the base rollout group is increased to $k=4$, $\text{Gap}(4) = 0.5$, and the testing radius is $1 / \sqrt{4} = 0.5$. The discrete gap identically matches the testing tolerance radius.
\end{itemize}

Continuing with the prior $V = 0.8$:
\begin{itemize}[noitemsep,topsep=0pt,leftmargin=*]
\item If all 4 rollouts are correct (mean $1$): the error is $|1 - 0.8| = 0.2 \leqslant 0.5$, accepting $V$.
\item If 3 rollouts are correct and 1 is incorrect (mean $0.5$): the error is $|0.5 - 0.8| = 0.3 \leqslant 0.5$, which still accepts $V$.
\end{itemize}

This critical threshold establishes a foundational statistical buffer zone. A single incorrect response generated by the policy model due to inherent variance is appropriately attributed to theoretical random noise, sustaining the validity of the prior and preventing the allocation of redundant compute.
For any $k \geqslant 4$, the system strictly satisfies the necessary condition where the tolerance radius comprehensively covers the discrete gap:
\begin{equation}
\frac{1}{\sqrt{k}} \geqslant \frac{2}{k} \implies k \geqslant 4
\end{equation}

Within this regime, the testing tolerance radius is sufficient to absorb at least one complete discrete jump, providing the hypothesis test with rigorous statistical validity.

Consequently, within the dynamic budget allocation framework integrating the value model prior and empirical hypothesis testing, the theoretical and engineering minimum initial rollout size is $k_{\text{min}} = 4$. Allocating an initial budget of less than 4 predictably induces a high false rejection rate due to the structural failure of the tolerance radius.
\end{proof}


\end{document}